\documentclass{article}

\usepackage{arxiv}

\usepackage[utf8]{inputenc} 
\usepackage[T1]{fontenc}    
\usepackage{hyperref}       
\usepackage{url}            
\usepackage{booktabs}       
\usepackage{amsfonts}       
\usepackage{nicefrac}       
\usepackage{microtype}      
\usepackage{lipsum}
\usepackage{fancyhdr}       
\usepackage{graphicx}       
\graphicspath{{media/}}     
\usepackage[numbers]{natbib}
\usepackage{doi}
\usepackage{amsmath,amsfonts}
\usepackage{amsthm}
\usepackage{cleveref}       
\usepackage{algorithmic}
\usepackage{algorithm}
\usepackage{adjustbox}
\usepackage{array}
\usepackage{textcomp}
\usepackage{stfloats}
\usepackage{url}
\usepackage{verbatim}
\usepackage{paralist}
\usepackage{graphicx}
\usepackage{bbm}
\usepackage{multirow}
\usepackage{amssymb}
\usepackage{subcaption}
\newtheorem{thm}{Theorem}

\newtheorem{lemma}{Lemma}
\DeclareMathOperator{\sech}{sech}

\title{\emph{LogGENE}:\\ A smooth alternative to check loss for \\ Deep Healthcare Inference Tasks}

\author{
  Aryaman Jeendgar\\
  BITS Pilani, Hyderabad Campus\\
  \texttt{jeendgararyaman@gmail.com}\\
\And
Tanmay Devale\\
BITS Pilani, Goa Campus\\
\texttt{f20190066@goa.bits-pilani.ac.in}\\
\And
Soma S Dhavala\\
ML Square\\
\texttt{soma.dhavala@gmail.com}\\
\And
Snehanshu Saha\\
BITS Pilani, Goa Campus\\
\texttt{snehanshus@goa.bits-pilani.ac.in}
}

\begin{document}
\maketitle

\begin{abstract}
Mining large datasets and obtaining calibrated predictions from them is of immediate relevance and utility in reliable deep learning.  In our work, we develop methods for Deep Neural Network based inferences in such datasets like the gene expression.  However, unlike typical Deep learning methods, our inferential technique, while achieving state-of-the-art performance in terms of accuracy, can also provide explanations, and report uncertainty estimates. We adopt the quantile regression framework to predict full conditional quantiles for a given set of housekeeping gene expressions. In addition to being useful in providing rich interpretations of the predictions, conditional quantiles are also robust to measurement noise. Our technique is particularly consequential in High-throughput Genomics, an area that is ushering in a new era in personalized health care, and targeted drug design and delivery. However, check loss, used in quantile regression to drive the estimation process, is not differentiable. We propose $log-cosh$ as a smooth alternative to the check loss. We apply our methods to the GEO microarray dataset. We also extend the method to the binary classification setting. Furthermore, we investigate other consequences of the smoothness of the loss in faster convergence. We further apply the classification framework to other healthcare inference tasks, such as heart disease, breast cancer, diabetes, etc. As a test of the generalization ability of our framework, other non-healthcare related data sets for regression and classification tasks are also evaluated.
\end{abstract}


\section{Introduction}
A quantitative study of gene expression and its underlying regulatory mechanism is of inherent value in  \emph{curing diseases} like heterogeneous tumours and in controlling protein production for biotechnology purposes \cite{natureArt}. This problem can be characterized in terms of understanding gene expression patterns of cells under various biological states. Diseases diagnosis and prognosis may be more broadly understood by measuring and observing changes in gene expression patterns. It might lead to better characterization of where, when, and how genetic instructions are decoded in diseased cells and tissues. Sizeable publicly available datasets have been made available such as the \emph{Connectivity Map} and the \emph{Tor-21} project \cite{tor21} towards such purposes. As a result, there is a surge in mining such gene expression datasets. A noteworthy application of this approach is to improve drug development success rates by screening vast libraries of compounds to affect and regulate the gene expression patterns so as to restore the conditions found in healthy tissue \cite{drugGEO}. Neural networks have been applied for inferring gene expressions \cite{GX6,GX7}, most notably in \cite{D-GEX}, and are used for modeling and simulations in drug development. 
\subsection{Our Approach}
The patterns that the neural networks learn depend largely on the loss function used to drive the training process. As a result, choosing an appropriate loss function is very crucial. With that goal in mind, we pick the \emph{Gene Expression} problem as a test bed for our experiments. We borrow the problem structure from \cite{D-GEX}, where the gene expression problem was framed to predict the so-called \emph{target genes} using already known \emph{landmark genes}. These are a selection of ~1000 genes that have been experimentally computed to be capturing up to $80\%$ of the information contained in the entire genome. \cite{LipGene} also used the same dataset to establish the utility of check loss. Hence, basing our experiments on this dataset to test our propositions would serve to \begin{inparaenum}
    \item benchmark our proposed loss for fair comparisons
    \item experimentally verify the theory of the $log-cosh$ that we construct in the rest of the paper, on an important, large, real-world dataset.
\end{inparaenum}

\section{Technical Motivation}
The Mean Absolute Error (MAE) is a loss function of classical importance \cite{huberRob} in robust regression settings and is often preferred to Mean Squared Error (MSE) under heteroscedastic measurement-error regimes. Its asymmetric counterpart, the check loss (or pinball loss), is used in quantile regression. Majority of modern Machine Learning techniques only focus on predicting the conditional means, whereas with check loss, one could infer the entire conditional distribution, which is useful in quantifying the aleatoric uncertainty in the predictions. However, one severe drawback of using vanilla MAE or its extension in practice, particularly with deep neural networks, is its \emph{non-differentiability}. With a renewed interest in exploiting first and second-order derivatives in meaningful ways, \cite{AdaH, PyHessian,nonCvx,infFunc}, it is pragmatic to look for smoother alternatives to MAE and its extensions. Specifically, we suggest \emph{$log-cosh$} as such an alternative to MAE and back this claim by providing relevant theoretical arguments and empirical validation on many data sets. Additionally,  we consider the \emph{Tilted $log-cosh$} rivaling the role of check loss in quantile regression. Later, \emph{Tilted $log-cosh$} is adapted to binary classification settings to demonstrate its efficacy in learning latent conditional quantiles. We show how, in disease predictive modeling, such latent quantiles can augment explainability. In summary, our contributions are:

\begin{enumerate}
    \item We prove relevant properties of the $log-cosh$ that immediately follow in section-\ref{section:4}. These include 1-Lipschitzness of $log-cosh$ that is used to demonstrate the utility in developing an Adaptive Learning Rate (LALR) training regime. This ensures faster convergence in several classification tasks for the loss function proposed in sections-\ref{section:4.7} and \ref{section:4.8}. The results are reflected in \ref{section:5.1.1}.
    \item We provide a differentiable alternative to MAE, while retaining its statistical robustness. The intuition for the argument is developed in section-\ref{section:4.1}.
    \item We exploit the higher-order differentiability of the $log-cosh$ by applying the L-BFGS optimizer on the GEO microarray dataset. The rationale, backed by our convexity proof, is explored in section-\ref{section:4.6}
    \item Extend $log-cosh$ to the quantile regression and binary classification settings. In section-\ref{section:4.7}, we propose a new loss function for use in binary classification settings, namely, the \textit{Smooth Binary Quantile Classification loss} which, in addition to being able to provide point predictions (like the Binary Cross Entropy), can also be used to quantify uncertainty in the predictions of the network.
\end{enumerate}

A salient observation is how $log-cosh$ lends itself to accomplish the above objectives simultaneously. As a result, our work lays an important foundation for building further. For example, \cite{infFunc} adapted Influence Functions (IFs) from robust statistics to deep learning setting. They can be used to fix labeling errors, provide counterfactuals, and identify out-of-distribution samples on gene networks, which could be extremely interesting. However, estimating IFs is fragile, and we hope that our work can alleviate some problems in this regard.

\section{Related Work}
Quantile regression \cite{QRkoenker} in neural network settings has been explored in recent works. \cite{QR3,QR4}. \cite{QR1} propose a constrained optimization approach for learning multiple non-crossing quantiles. \cite{QR2} introduced a Bayesian neural network for quantile regression based on Asymmetric Laplace Distribution (ALD). Notice that, in our work, when adapting the $log-cosh$ to the binary classification setting, the response variables are modeled using the hyperbolic secant distribution, instead of ALD. \\ \cite{D-GEX} propose a deep learning approach for gene expression inference, known as D-GEX. They posit the inference as a multi-task regression problem and considered a multi-layer feed-forward neural network, trained with MSE loss. Since then, there have been multiple works exploring the application of neural networks to the problem of gene expression inference such as \cite{GX6,GX7,GX8}. On the GEO microarray dataset, \cite{LipGene} have applied quantile regression in the neural network setting, with Lipschitz Adaptive Learning rates. In our work, we recreate a subset of the results with tilted $log-cosh$ (that offers a smooth alternative to the loss in \cite{LipGene}) on the same dataset.  \cite{BQR} adapt the check loss to the binary classification setting. However, we show that tilted $log-cosh$ may also be adapted to the binary classification setting. Additionally, recent work \cite{LCvae} has explored the applicability of the $log-cosh$ to the generative domain as well.\\ The deep learning community has seen a recent resurgence in second-order optimization methods, with the emergence of better compute and tractability options, which were previously infeasible \cite{AdaH,PyHessian,nonCvx,infFunc}. In our work, we show the applicability of L-BFGS \cite{LBFGS}, a standard second-order optimizer, on $log-cosh$.

\section{Theoretical framework} 
\label{section:4}
This section presents some theoretical results that present the $\log-\cosh$ as a viable loss function in the deep learning setting. Most of these results complement the simple 'interoperability' (between the MSE and the MAE) of the $\log-\cosh$ that we presented above. section-\ref{section:4.1}, makes a case for the $\log-\cosh$ in comparison to its closest competitors (namely, the MSE, MAE, and the Huber), sections-\ref{section:4.2},\ref{section:4.3},\ref{section:4.4} present some essential results regarding the convexity and lipschitzness (which is later used in the adaptive learning rate training regimens). Section-\ref{section:4.6} presents an example application of the $\log-\cosh$ which exploits its double-differentiability, section-\ref{section:4.7} presents our extension of the $\log-\cosh$ to the binary classification setting, with section-\ref{section:4.8} and section-\ref{section:4.9} providing results for lipschitzness and relevant regularization for making learning multiple quantiles with the classification loss feasible.
\subsection{Why the \texorpdfstring{$\log-\cosh$}{Lg}}
\label{section:4.1}
The fundamental intuition behind using the $log-cosh$ as a plausible alternative to the MSE and MAE is due to the following elementary observation. 
Consider:
\begin{align*}
\log(\cosh(x))&=\log(\frac{e^x+e^{-x}}{2}) \\
              &=\begin{cases}
      \mid x\mid-\log(2) & \text{large x} \\
       \frac{x^2}{2} & \text{small x}  \\
   \end{cases}
\end{align*}
This, in essence, tells us that the $log-cosh$ behaves like the MAE asymptotically, which is desirable in applications such as \emph{robust regression} where the influence exerted by outliers in response is limited. Additionally, around the origin, $\log-\cosh$ behaves like MSE, retaining the statistical efficiency when measurement errors follow Normal distribution.
\begin{figure}
     \centering
     \begin{subfigure}[b]{0.4\textwidth}
         \includegraphics[width=\textwidth]{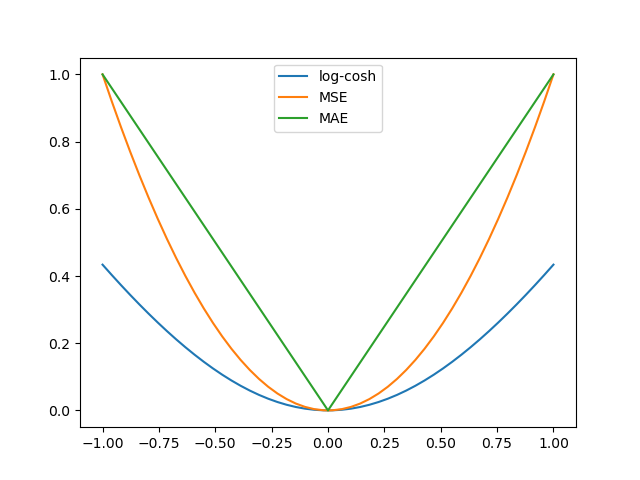}
         \caption{MSE-like behaviour of the $\log-\cosh$}
     \end{subfigure}
     \hfill
     \begin{subfigure}[b]{0.4\textwidth}
         \includegraphics[width=\textwidth]{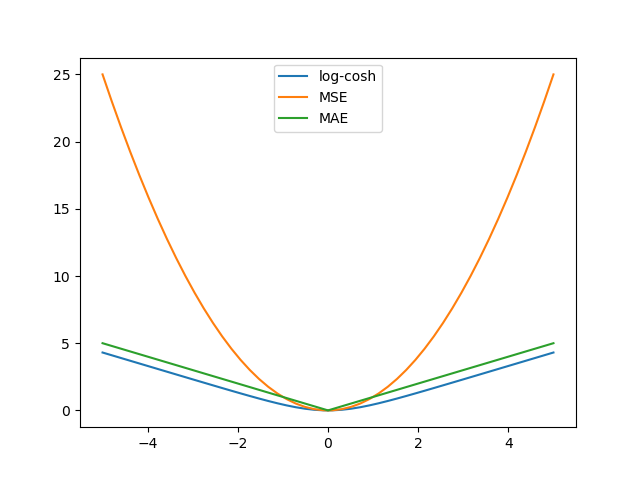}
         \caption{MAE-like behaviour of the $\log-\cosh$}
     \end{subfigure}
\end{figure}

Another useful property of $\log-\cosh$ that we demonstrate is its \emph{robustness to label noise}, meaning that the fluctuations in function output under small perturbations in its input have an upper bound. This places $\log-\cosh$ on an equivalent setting as MAE. This is formalized in section-\ref{app:robust} of the appendix.

\noindent The above also captures our line of interest in the $log-cosh$ as compared to \emph{MSE}, \emph{MAE} or its 'closest competitor', the \emph{Huber} loss, which is defined as:
\[
L_{\delta}(y,\hat{y})=\begin{cases}
      \frac{1}{2}(y-\hat{y})^2 & |(y-\hat{y})|\leq\delta \\
      \delta(|(y-\hat{y})|-\frac{1}{2}\delta) & \text{otherwise} \\
   \end{cases}
\]
The Huber 'intuitively' combines the MAE and MSE in the same sense but in a piece-wise manner. Like Huber, $log-cosh$ can also be further parameterized, if desired, to control the transition from MSE to MAE-like behavior. Consider $\log-\cosh(\frac{x}{h})$, where $h$ can be tuned as per the needs of the application. For instance, if one wants more MAE-like behavior, and the natural logarithm is used, then $h\approx1$ is close to optimal, whereas, for more MSE-like behavior, $h\approx0.7$ is close to optimal. Both MSE and MAE can be modeled as special cases. Hence, in that sense, it captures the same kind of convenience as the Huber in offering us a 'smooth' transition between the MAE and the MSE. At the same time, the $log-cosh$ has the following properties that make it more appealing than Huber in the context of our applications:
\begin{enumerate}
    \item The $log-cosh$ has a tractable and well-studied generating distribution -- the \emph{hyperbolic secant distribution} -- which allows a very convenient extension of the $log-cosh$ to the non-parametric setting. Thus, adapting it for use in binary classification problems using neural networks becomes convenient.
    \item Unlike the Huber, the $log-cosh$ is arbitrarily differentiable \emph{globally} (unlike the limited second-order differentiability of the Huber in the '$\delta$-basin' of its definition). Hence, the application of methods exploiting higher-order differentiability is inevitably more stable for the $log-cosh$.
\end{enumerate}
\subsection{Convexity of the loss}\label{section:4.2}
In this section, we argue that the $log-cosh$ is convex and defer the proof to section - \ref{app:thm1} of the appendix. Convexity of the loss is an immediately desirable quality because of the standard property of convex functions having a unique global minima and being much more "well-behaved" (Convexity implies local Lipschitzness). The latter is significant in ensuring a predictable, 'smooth' trajectory of the optimizer while navigating through the generated loss landscape. 

\begin{thm}[$\log-\cosh$ is convex]
i.e. \\ $J=\Sigma_{i=1}^m\log\cosh(y_i-\theta^Tx_i)$ is convex.
\end{thm}
The proof of the above theorem can be found in Section-\ref{app:thm1}. Our proof proceeds by first constructing the \emph{Hessian} corresponding to the loss, and proving that it is \emph{positive-definite}. Convexity guarantees the embedding of $\log-\cosh$ in L-BFGS and in general, second-order optimizers as the constructed Hessian is positive definite ensuring positivity of the eigenvalues of the Hessian, $J$.
\subsection{Lipschitzness}\label{section:4.3}
We would be interested in the Lipschitzness of the $\log-\cosh$ from the perspective of a useful mathematical property and also to demonstrate the utility in developing an \emph{Adaptive Learning Rate (LALR)} training regime. The latter builds on the Lipschitzness of the function \cite{LALR} which is a consequence of the convexity of $\log-\cosh$. In this section, we prove that the $\log-\cosh$ is Lipschitz, and derive a tighter bound for the Lipschitz constant that we use in later experiments.
\begin{lemma}[Lipschitzness of $\log-\cosh$]
$\log-\cosh$ is \emph{at least} 1-Lipschitz
\end{lemma}
\begin{proof}
It can be easily shown that, if the derivative of a differentiable function is bounded in a given domain by some $\rho$, then the function is \emph{$\rho$-Lipschitz} -- this follows directly from the definition of Lipschitzness and the \emph{Mean-Value-Theorem}.
Now, $\frac{d}{dx}\log\cosh(x)=\tanh(x)$.
Since $|\tanh(x)|\leq 1$, $\log-\cosh$ is \emph{at least} 1-Lipschitz.
\end{proof}

\subsection{A tighter Lipschitz constant}\label{section:4.4}
Let us derive a more meaningful and \emph{tighter} Lipschitz constant for the $\log-\cosh$. It should be noted that the Lipschitz constant is derived for the $\log-\cosh$ specifically in the \emph{neural network} (regression) setting.
\begin{thm}[Tighter Lipschitz Constant of the $\log-\cosh$]
$\log-\cosh$ is Lipschitz, with a Lipschitz constant:
\[
\frac{1}{m}\tanh(g(0)-\mid\mid y\mid\mid).\max_ja_j^{[L]}
\]
\end{thm}
The proof of the above theorem can be found in section-\ref{app:thm2}

\subsection{Applications to second-order Optimization}\label{section:4.6}
Gradient descent is the de-facto optimization method for neural network training, but there has been a recent interest \cite{AdaH,PyHessian} in adapting some classical second-order methods to neural network optimization. The largest barrier to using second-order optimization techniques (by and large \emph{Line-Search methods} \cite{nocedal2nd}) are:
\begin{itemize}
    \item Repeated computation of \emph{Full} Hessian during backpropagation is not feasible.  This can be alleviated to a certain extent by the use of \emph{approximate} computations of the Hessian. The class of quasi-Newton methods such as \emph{L-BFGS} \cite{LBFGS} aims to do exactly that, while still retaining the functional form of the Newton update.
    \item Noise in the hessian computations makes the use of Newton updates in the case of neural networks impractical. For first-order methods, there are several practical and well-developed methods for dealing with noisy gradient estimates like \emph{Adam}. Some standard practices in neural network training like \emph{mini-batching} are also available. There has been some recent work in this direction, \cite{AdaH} is one promising second-order optimizer which replaces the gradient in the Adam update with the approximation of the diagonal of the Hessian.
\end{itemize}

As stated earlier, one of our objectives is also to test the efficacy of $\log-\cosh$ in second-order optimizers, for fitting the models. Is $\log-\cosh$ a good candidate comparable to MSE? At the outset, MSE appears like an ideal candidate for the application of second-order optimizers, since line search methods in the case of quadratic objective functions are exact. However, this may not be the case with deep neural networks because of the inherent stochasticity and multi-modality, a very plausible scenario in the loss landscape. Multi-modality, in the context of optimization, is defined here as the possibility of having multiple local minima and the absence of global minima in non-convex loss surfaces.

With this line of experiments, we want to show that the $\log-\cosh$ can handle neural network training regimes that require higher-order information about the loss landscape. Our baseline for comparison as mentioned earlier would be the MSE (which is an ideal candidate for the application of second-order methods because the Newton step is \textit{exact} for quadratic functions)

We demonstrate the above by training the network using a popular choice from the family of \emph{quasi-Newton methods} \cite{nocedal2nd} (which are a group of algorithms that approximate the \textit{Newton update} by constructing an approximation to the Hessian instead of computing the full Hessian), namely, the \textbf{L-BFGS} (Limited-Memory BFGS) algorithm.

Any such approximation to the Hessian needs to satisfy the so-called \textit{secant equation}, namely:
\begin{equation*}
    B_{k+1}s_{k}=y_{k}
\end{equation*}
where, $s_{k}=x_{k+1}-x_{k},\qquad y_{k}=\nabla f_{k+1}-\nabla f_{k}$.\\
The BFGS update to the Hessian approximation (denoted by $B_{k}$) is constructed iteratively and satisfies the above condition which is:
\begin{equation*}
    B_{k+1}=B_{k}-\frac{B_{k}s_{k}s_{k}^{T}B_{k}}{s_{k}^{T}B_{k}s_{k}}+\frac{y_{k}Y_{k}^{T}}{y_{k}^{T}s_{k}}
\end{equation*}
We choose to experiment with \emph{L-BFGS} since
The functional form of the update that the method uses is the same as the \emph{Newton Update}, i.e. a step in the direction: $p_k=-B_k^{-1}\nabla f_k$, where $B_k$ is the approximate Hessian constructed by the algorithm \cite{nocedal2nd}. Intuitively, L-BFGS ought to converge in a few steps in this direction (for the cases of quadratic objective functions), bringing us close to the one-step convergence in the Newton update.
In section \ref{section:5.1.1}, we show through experiments with L-BFGS and a 'medium' neural network architecture, that the $\log-\cosh$ performs significantly better than the MSE. Additionally, significantly less over-fitting can be seen in figure-\ref{fig:Figure-1}. On classification and regression tasks, $\log-\cosh$ is at least on par with other benchmark loss functions and optimizers.

\subsection{Extension to the binary-classification setting}\label{section:4.7}
It is important to show that our proposed loss function can be readily extended to the binary classification setting. The following theorem formalizes the classification marker and construct of the tilted $\log-\cosh$.
\begin{thm}
The \textbf{Smooth Binary Quantile Classification Loss} derived from the $log-cosh$ is:
\begin{align*}
L(y_i, \hat{y_i})&= y_i\log(\hat{p_i})+(1-y_i)\log(1-\hat{p_i})\\
\hat{p_i}&=1-F_{\tau}(\hat{y_i})\\
\hat{y_i}&=f_{\tau}(x_i) \quad \text{ where, $f_{\tau}$, is the latent function}
\end{align*}
For any real-valued random variable $Z$, with distribution function $F(z)$, with $F(z) = P(Z \le z) $, the quantile function $Q(\tau)$ is given as
$Q(\tau) = F^{-1}(\tau) = \inf\{r: F(r) \ge \tau\} $ for any $0 < \tau < 1 $. Define $\tau$ as the marker for a typical quantile loss. Then, the CDF, $F_{\tau}(\cdot)$ of $f$ assumes the closed-form expression:
\begin{equation*}
F_{\tau}(x)= \begin{cases}
\tau+\frac{4\tau}{\pi}\tan^{-1}(\tanh(\frac{x}{2})) & x\leq 0\\
\tau+\frac{4(1-\tau)}{\pi}\tan^{-1}(\tan(\frac{x}{2})) & x>0
   \end{cases}
\end{equation*}
\end{thm}

The proof of the above theorem can be found in section-\ref{app:thm3} of the Appendix.
Note: $\tau$ is the marker for quantile losses which can be extended from $logcos{h(x)}$, making this loss interpretable as well. For example, $\tau=0.5$ gives us the median i.e. MAE and $L, L', L"$ are well-defined for $\tau=0.5$. Thus, at $\tau=0.5$, the smooth, quantiled version of MAE i.e $logcos{h(x)}$ is continuous and twice differentiable and interpretable in the sense of binary quantile regression and binary quantile classification \cite{BQR}
\subsection{Lipschitzness of the sBQC loss}\label{section:4.8}
\emph{Smooth Binary Quantile Classification Loss}, defined above, is also Lipschitz. In the following theorem, we provide a tighter bound for the Lipschitz constant, which is used in \emph{LALR} training regime. 
\begin{thm}[Lipschitz constant for the \emph{sBQC}]
The Binary Smooth Quantile Classification Loss has the Lipschitz constant:
\[
\frac{2}{\pi} \max\left(1, \frac{1-\tau}{\tau}, \frac{\tau}{\tau-1}\right)
\]
\end{thm}
The proof of the above theorem can be found in section-\ref{app:thm4} of the Appendix

\subsection{Regularization}\label{section:4.9}
Finally, we describe a penalty term that needs to be added to classification/regression loss. To that end, we specify that our model can be rewritten as:
\[
y=I(z\geq 0), z=Q_x(\tau)=f_{\tau}(x)+\epsilon
\]
where $\epsilon\sim HSD(\tau)$ and $HSD(y;\tau)=\frac{2}{\pi}\sech(y)[\tau. I(y<0)+(1-\tau).I(y\geq 0)]$. The network learns the underlying latent function represented by $Q_x(\tau)$. When multiple quantiles are fit separately, it may be possible for a lower quantile to be greater than an upper quantile -- known as quantile crossing. The following regularizing term penalizes such quantile crossing:
\[
L_{reg}=\sum_{i=1}^{n}\sum_{p=1}^{m-1}\max(0,Q_{x_i}(\tau_p)-Q_{x_i}(\tau_{p+1}))
\]
\section{Experimental results}
The code and data for reproducing the results may be found at \cite{C&D} and \cite{geodata}. Code for different architectures is also uploaded as a supplementary file.

\subsection{\emph{GEO Microarray} dataset}\label{section:5.1}

\begin{table}[ht]
\centering
\begin{adjustbox}{width={0.5\textwidth}, center}
\begin{tabular}{ccll}
\hline
Model name & Network size & Dropout \\ \hline
Small architecture & [300,300] & 10\% \\ \hline
Medium architecture & [1000,1000] & 10\% \\ \hline
Wide Medium architecture 1& [2000] & 20\% \\\hline
Wide Medium architecture 2& [2000] & 10\% \\\hline
Large architecture & [3000,3000] & 10\% \\ \hline

\end{tabular}
\end{adjustbox}
\caption{Model architectures}
\label{table:architectures}
\end{table}

\noindent \emph{Training methodology:} We run tests on five architectures. Three of them have 2-layer fully-connected neural network architectures which differ only in the size of their hidden layers. The 'small', 'medium', and 'large' architectures have [300,300], [1000,1000] and [3000,3000] as the size of their hidden layers respectively, with a dropout of $0.1$ for each and two 1-layer fully-connected neural network architecture 'wide-medium' with [2000] as the size of its hidden layer and one with a dropout of $0.2$ and another with a dropout of $0.1$ (as defined in \cite{LipGene}). The dataset has 943 \emph{landmark genes}, which serve as the input features, and 4760 \emph{target genes}, which are modeled as the outputs. We compare \emph{$log-cosh$} against \emph{MSE} and \emph{MAE}. Each model was run on different training regimes, namely, \emph{Constant Learning Rate with Adam}, \emph{Lipschitz Adaptive Learning Rate with Adam} and \emph{L-BFGS} (only with $log-cosh$ and MSE). We train two fully-connected neural network architectures for all of the runs, differing in the sizes of their intermediate layers. Each experiment was run for 500 epochs each.\\
\subsubsection{Results}\label{section:5.1.1}
The results yielded by our experiments allow us to draw insights into the superior performance of $log-cosh$. Firstly, networks trained with Log-cosh and MSE consistently outperform MAE. Secondly, our expectation to have L-BFGS converge in fewer iterations compared to its first order counterparts because of the quadratic nature of the $log-cosh$ (section-\ref{section:4.6}) is well-supported by the results in table-\ref{table:GEO} We observe superior performance of $log-cosh$ with LFBGS in comparison to MSE for the medium architecture and comparable performance when deployed on the small architecture. We extend the experiments on three other architectures mentioned in \cite{LipGene} namely Wide-medium 
(1 \& 2) and Large (see table \ref{table:architectures}). The performance of $log-cosh$ is either comparable or (marginally) superior to MSE but definitely superior to the MAE($L_1$) loss. We further experimented with additional hidden layers while keeping the number of neurons in each layer constant. We do not observe any significant improvement in performance (of all the three losses) on deeper architectures.

Additionally, for several datasets such as Heart Disease, Ionosphere, Pima, and WBC, the LALR sBQC model converges significantly faster than their non-adaptive counterparts. For WBC, the LALR sBQC converges 11 times faster than CLR SBQ. For Heart disease, Haberman, Ionosphere, and Wisconsin Breast Cancer (WBC), we also observe that the sBQC converges faster to a certain accuracy than BCE. We achieved this result by setting the maximum accuracy of CLR sBQC as the threshold and letting the LALR version run till sBQC LALR achieves the threshold obtained by its fixed LR counterpart. 

Finally, another interesting observation highlighted by the figure-\ref{fig:Figure-1}, is that MSE overfits to a substantially greater extent than $log-cosh$ while using L-BFGS as our optimizer. A rigorous undertaking to investigate this aspect is deferred for future work. Nonetheless, we observe the superior performance of $log-cosh$ on 'medium' architectures in terms of RMSE and data sets and at least, equivalent performance as the baselines provided for by the MSE and MAE results. These tests validate that the $log-cosh$ is a valid (and sometimes better) alternative to both the MSE and MAE.
 \begin{table*}[ht]
 \centering
\begin{adjustbox}{width={\textwidth}, center}
\begin{tabular}{cclllll}
 \hline
 Loss & Optimizer & D-GEX(small) & D-GEX(medium) & D-GEX(wide medium 1) & D-GEX(wide medium 2) & D-GEX(large)  \\ 
 \hline
 \multirow{3}{*}{$log-cosh$} & Adam &0.6348 &0.6777 & 0.8112& 0.7963& 0.8750 \\ \cline{2-7} 
                         & LALR-Adam &0.6355  &0.6054 & \textbf{0.6670}&\textbf{0.6632}& \textbf{0.6349} \\ \cline{2-7} 
                           & LBFGS     &\textbf{0.5642} &\textbf{0.5490}& 0.8165& 0.7953 & 0.8706\\ \hline
\multirow{3}{*}{MSE}      & Adam  &0.6321  &0.6809&0.8160 & 0.8057&0.8741\\ \cline{2-7} 
                          & LALR-Adam &0.5937 &\textbf{0.5389}& \textbf{0.6512}& \textbf{0.6505} &\textbf{0.6449}\\ \cline{2-7} 
                          & LBFGS  &\textbf{0.5582} &0.6273&0.8217 & 0.8056 &0.8697\\ \hline
 \multirow{2}{*}{MAE}      & Adam &0.6528 &0.7152& 0.8158&0.8400 &0.8990 \\ \cline{2-7} 
                          & LALR-Adam &0.6520 &0.6261& 0.8830&0.8822 &0.9250 \\ \hline
\multirow{2}{*}{Check Loss}      & Adam  &0.6530  &0.7155&0.8164 & 0.7961&0.6540\\ \cline{2-7} 
                          & LALR-Adam &0.9308    &0.9293 & 0.9747& 0.9751&0.9314\\ \hline
\multirow{2}{*}{Huber Loss}      & Adam  &0.7214  &0.7448 &0.8376 & 0.8465 &0.8887\\ \cline{2-7} 
                          & LALR-Adam &0.9270    &0.9225 & 0.8764 &0.8789 &0.9261\\ \hline                          
 \end{tabular}
 \end{adjustbox}
 \caption{RMSE on GEO microarray: L-BFGS doesn't apply to MAE, Check and Huber Losses}
 \label{table:GEO}
 \end{table*}

\begin{table*}[ht]\label{Table-1}
\renewcommand{\arraystretch}{1.5}
\centering
\begin{adjustbox}{width={\textwidth}, center}
\begin{tabular}{lllllllll|llllllll}
\hline
\multirow{3}{*}{Name} & \multicolumn{8}{c}{sBQC}                                 & \multicolumn{8}{c}{BCE}                                  \\ \cline{2-17} 
                      & \multicolumn{4}{c}{Adam} & \multicolumn{4}{c}{LALR-Adam} & \multicolumn{4}{c}{Adam} & \multicolumn{4}{c}{LALR-Adam} \\ \cline{2-17} 
                      & CP  & JI  & F1  & \%Acc. & CP   & JI   & F1   & \%Acc.   & CP  & JI  & F1  & \%Acc. & CP   & JI   & F1   & \%Acc.   \\ \hline
Heart Disease &\textbf{0.6553}&\textbf{0.7162}&\textbf{0.8346}&\textbf{82.78}&\textbf{0.5227}&\textbf{0.6506}&\textbf{0.7883}&\textbf{76.22}&0.6058&0.6883&0.8153&80.32&0.5081&0.6103&0.7581&75.41\\ \hline
WBC                    &\textbf{0.9067}&\textbf{0.8867}&\textbf{0.94}&\textbf{95.71}&\textbf{0.9304}&\textbf{0.9151}&\textbf{0.9556}&\textbf{96.78}&0.8906&0.8679&0.9292&95.00&0.9154&0.8981&0.9463&96.07          \\ \hline
Pima                  &\textbf{0.4005}&0.3790&0.5497&75.00&0.4215&0.4716&0.6410&\textbf{76.29}&0.3792&0.4238&0.5953&73.37&0.4411&0.4840&0.6523&73.70   \\ \hline
Titanic 
&0.6996&0.6872&0.8165&\textbf{85.30}&0.2471&0.3662&0.5283&64.31
&0.7153&0.6946&0.8278&86.06&0.6652&0.6627&0.7965&83.9306\\ \hline
Haberman         &\textbf{0.2067}&0.2142&0.3529&79.67&0.2548&0.2558&0.4074&81.30&0.1787&0.2173&0.3571&80.48&0.3045&0.2631&0.4167&79.67\\ \hline
Ionosphere            &0.6802&0.8241&0.9035&86.42&\textbf{0.7031}&0.81&0.8950&\textbf{86.42}&0.7193&0.8381&0.9119&87.85&0.7005&0.8118&0.8961&86.42\\ \hline
Sonar                 &\textbf{0.5821}&\textbf{0.6964}&\textbf{0.8210}&\textbf{79.51}&\textbf{0.6186}&0.6923&\textbf{0.8181}&\textbf{80.72}&0.5556&0.6842&0.8125&78.31&0.4792&0.6441&0.7835&74.69          \\ \hline
Banknote              &\textbf{1.0}&\textbf{1.0}&\textbf{1.0}&\textbf{100.0}&\textbf{0.9927}&\textbf{0.9921}&\textbf{0.9960}&\textbf{99.63}&1.0&1.0&1.0&100.0&0.9926&0.9921&0.9960&99.63\\ \hline
\end{tabular}
\end{adjustbox}

\caption{Binary Classification Results:  Cohen’s Kappa (CP), Jaccard index (JI),
F1-score (F1) and Accuracy ($\%$Acc.)}
\end{table*}

\begin{table*}[ht]
\centering
\renewcommand{\arraystretch}{1.25}
\begin{tabular}{llll|lll|ll}
\hline
\multirow{2}{*}{Name} & \multicolumn{3}{c}{$Log-cosh$}                                                         & \multicolumn{3}{c}{MSE}                                                              & \multicolumn{2}{c}{MAE}                                                              \\ \cline{2-9} 
                      & \multicolumn{1}{c}{Adam} & \multicolumn{1}{c}{LALR-Adam} & \multicolumn{1}{c}{L-BFGS} & \multicolumn{1}{c}{Adam} & \multicolumn{1}{c}{LALR-Adam} & \multicolumn{1}{c}{LBFGS} & \multicolumn{1}{c}{Adam} & \multicolumn{1}{c}{LALR-Adam} \\ \hline
Abalone               &3.3878&3.149&\textbf{3.14386}&3.3269&3.1982&\textbf{3.1232}&3.5371&3.1979                           \\ \hline
Boston                &10.7078&9.9378&\textbf{9.8665}&9.7432&9.7171&\textbf{9.7173}&9.8886&10.1969                          \\ \hline
Concrete              &17.1732&17.6689&\textbf{17.6579}&17.9485&17.6132&\textbf{17.62085}&17.9380&17.6980                           \\ \hline
Energy                &10.8102&10.3885&\textbf{10.2397}&11.6565&10.2407&\textbf{10.2499}&10.3692&10.3747\\ \hline
Wine                  &1.083&0.8673&\textbf{0.8371}&1.0353&0.8590&0.8367&1.0783&0.9311                           \\ \hline
\end{tabular}
\caption{Generalizing L-BFGS to UCI Regression data: L-BFGS not applicable to MAE (Section-\ref{section:5.1}, RMSE values)}
\label{table:2}
\end{table*}

\begin{figure}[ht]
\centering
\includegraphics[width=0.4\textwidth]{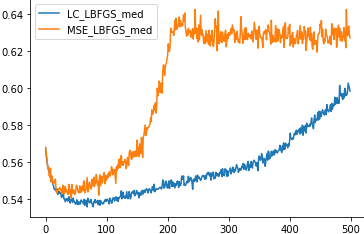}
\caption{RMSE vs epochs for MSE and $log-cosh$ trained using L-BFGS on the medium architecture. MSE overfits.}
\label{fig:Figure-1}
\end{figure}

\subsection{Generalizability to other tasks}
We also performed generalizability tests (to other tasks) on some standard UCI regression datasets to show that our method generalizes beyond just the GEO microarray dataset. These are collated in table-\ref{table:2}. Root Mean Squared Error (RMSE) on validation sets of various UCI regression datasets is reported. $Log-cosh$ is compared against MSE and MAE while using three different optimizers namely, Adam (with constant learning rate), Adam (with LALR), and L-BFGS. 
Clearly, we observe that the $log-cosh$ performs just as well as the MSE and MAE baseline results.

\subsection{Binary Classification Problems}
We use a standard three-layer neural network, with a hidden layer of size 100 for training the UCI binary classification datasets. We report \emph{Cohen's Kappa} (CP), \emph{Jaccard index} (JI), \emph{F1-score} and Accuracy. The \emph{BCE} is used as the baseline for comparison against the proposed method, \emph{sBQC}. \emph{sBQC} exhibits superior performance on most of the classification tasks.

\subsubsection{Model Evaluation}
5-fold cross-validation is adopted on the data. We split the data into training and validation sets (80-20). These sets are selected such that their class distribution is representative of the original dataset (stratified k-fold). The training set is split into five folds (one fold is used as the test set), with the validation set being kept separate. This technique ensures that the models are not biased and the results are generalizable. Each fold was scrutinized carefully to check for no data leakage from the train set to the validation set, on all the runs. For all our experiments, we use the free version of Google colab, with 1 K80 GPU, 2 vCPUs, and 12GB RAM.
\par Each experiment is run 20 times and the model is trained for 50 epochs in each run. The mean of the metrics under consideration is reported. We didn't report standard deviation values from the mean accuracy. The consistent performance of sBQC loss is apparent from the standard deviation in accuracy. The standard deviation values range from $\pm0.17$ to $\pm0.37$ which is reasonably small. The performance of sBQC and BCE was cataloged over twenty runs and we find that on the Pima Indian diabetes, Wisconsin Breast Cancer(WBC), Cleveland Heart Disease, and Haberman's survival datasets, sBQC achieves better accuracy (the optimizer used for these runs was Adam with LALR).

\subsection{Comparison of different learning paradigms with LALR}

We ran experiments on the GEO Microarray dataset using the small, medium, wide medium 1, wide medium 2 and large architectures as explained in \ref{table:architectures}. We have compared the performance of LALR with $log-cosh$ with a constant learning rate of 0.1 using Adam's optimizer(results of Adam with learning rate of 0.01 in Table 2) and an exponential decaying learning rate given by \( 0.9\times e^{0.0001*\text{epoch}}\). We observe that $log-cosh$ consistently outperforms the other baseline paradigms on all architectures.
\begin{table*}\label{table5}
 \centering
\begin{adjustbox}{width={\textwidth}, center}
\begin{tabular}{cclllll}
 \hline
 Optimizer & loss & D-GEX(small) & D-GEX(medium) & D-GEX(wide medium 1) & D-GEX(wide medium 2) & D-GEX(large)  \\ 
 \hline
 \multirow{3}{*}{Adam(0.1)} & Check Loss & 1.1194 & 2.2694 & 0.9112 &0.9172& 6.1569 \\ \cline{2-7} 
                         & MSE & 1.0118  & 1.5406 & 0.8968&0.9022& 4.0935 \\ \cline{2-7} 
                           & MAE     &1.1244 & 2.2837 & 0.9121& 0.9185 & 6.2656\\ \hline
\multirow{3}{*}{Adam(Exponential Decay)}      & Check Loss  &0.7329  &0.7063&0.9722 & 0.9778&0.6922\\ \cline{2-7} 
                          & MSE &0.7650 &0.7346& 0.9500& 0.9529 &0.6997\\ \cline{2-7} 
                          & MAE  &0.7305 &0.7103&0.9737 & 0.9791 &0.6900\\ \hline
LALR & $log-cosh$  &\textbf{0.6355}  &\textbf{0.6054} & \textbf{0.6670}&\textbf{0.6632}& \textbf{0.6349} \\ \hline
 \end{tabular}
 \end{adjustbox}
 \caption{Comparison of different learning paradigms with LALR}
 \end{table*}
 
\subsection{Quantiles and Interpretability}
This section aims to demonstrate a practical example of the explainability that learning multiple quantiles can offer. Using the notation we introduced in \ref{section:4.9}, in a binary classification problem, we would be interested in observing for values of input the make latent function zero i.e. $z=0$ (since that specifies our decision boundary). That involves solving the equation $Q_x(\tau)=0$, in $\tau$, for a fixed $x$. Assuming that we are able to solve the above equation for \emph{any given x} we can make an important claim:\\ \emph{Given $x$ being the value that the variable of interest takes on, and for the same $\tau$ being the solution to $Q_x(\tau)=0$, at $x$, there is a $\tau\%$ chance that $x$ would be classified as $0$, and $(1-\tau)\%$ chance as $1$}. 
\begin{figure}[ht]
\centering
\includegraphics[width=0.4\textwidth]{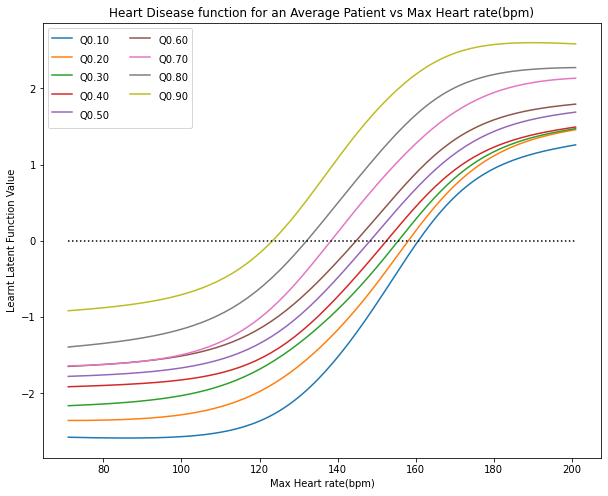}
\caption{Multiple quantiles, heart-rate dataset}
\label{fig3}
\end{figure}

\begin{figure}[ht]
\centering
\includegraphics[width=0.4\textwidth]{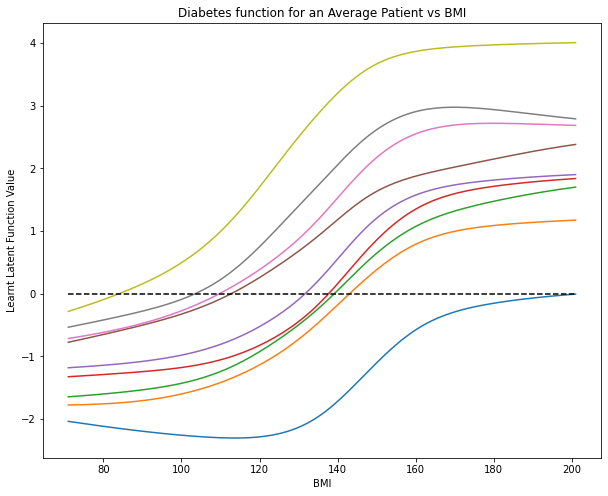}
\caption{Multiple quantiles, Pima Indian dataset}
\label{fig4}
\end{figure}
For instance, we produce the curve for the Heart Rate dataset where we used the $log-cosh$ to learn the quantiles. An example reading of the curve would be: If there are two patients whose heart rates are 120 BPM and 140 BPM, then there is a 90\% and a 70\% probability respectively, that the patients do not have heart disease (Fig.\ref{fig3}). Similar trends are observed for the Pima Diabetes data set (Fig. \ref{fig4}).
These explanations are ante-hoc, baked-in to the model itself, do not require any additional post-processing, unlike standard XAI techniques like shaplay and LIME.

\section{Discussion and Conclusion} Can gene expression inference techniques lead to cheaper drug design? A good precedence has been set already with leading drug manufacturers like Pfizer and Genentech using AI-powered solutions in their search for immuno-oncology drugs and cancer treatments. Efficient and accurate inference techniques are likely to produce quicker, cheaper, and more effective drug discovery leading to equitability in healthcare \cite{drugGenX}. We foresee the technique developed in the paper being used in a similar manner as DeepCE, an efficient Deep Learning-based inference model \cite{COVIDgenX} which predicts correlations between gene expression and drug response. The method has helped identify drug re-purposing candidates for COVID-19 out of which two drugs (cyclosporine and anidulafungin) have received regulatory approval from the FDA.\\
The UN Department of Social and Economic Affairs suggests 17 goals \cite{UNsdg} which offer a systematized framework to achieve more equitable conditions of living. Developing computational and analytical models of gene regulations can contribute to one of the goals, which is to provide and improve \emph{Good-Health and Well-Being}. Adopting a deep learning approach to this problem allows us to exploit a large amount of gene expression data. In addition, making these black-box models more explainable and robust is a significant step in the direction of harnessing the potential of gene expression dataset.

Our paper studies the analytical properties of well-known $log-cosh$ and extend it to estimate conditional quantiles to binary classification settings. On the former,  we provided convexity guarantees of $log-cosh$ which make their application in second order optimizers such as LBFGS and also in ADAM, more grounded in theory. We prove Lipschitzness of $log-cosh$ and used it to demonstrate the utility in developing an Adaptive Learning Rate (LALR) training regime. This ensures non-trivial, automated learning rate computation and adaptation. On the later, we propose a new loss function for use in binary classification settings, namely, the Smooth Binary Quantile Classification loss which, in addition to being able to provide point predictions (like the Binary Cross Entropy), can also be used to quantify uncertainty in the predictions of the network. Most importantly, we showed that our proposed loss function for classification, sBQC, can be implemented on a second order optimizer while retaining the robustness of MAE, Check Loss and Huber Loss.
We have carried out extensive experiments to validate the claims made in the paper. Our experiments suggest that the $log-cosh$ can be considered as an alternative to the MAE that can achieve state-of-the-art performance. As we briefly discussed in the introduction, we are able to show via our experimentation on the GEO microarray dataset that the above-stated properties of the $log-cosh$ make it a model candidate for working with such sensitive data owing to the additional interpretability that we get out of quantiles. The benefits of such robust and smooth loss can be compounded in straightforward ways with the application of cutting-edge higher-order methods. \\In summary, $log-cosh$ accomplishes multiple objectives simultaneously in leveraging smoothness toward adaptive learning rates, providing smooth and robust alternatives to classical loss functions while offering higher-order tractable alternatives to standard optimizers. Additionally, its binary classification analogue, the sBQC loss has been shown to achieve faster convergence due to its Lipschitzness, even against the similarly augmented BCE. We observe further that, $log-cosh$ with L-BFGS has RMSE 0.5642 and 0.5490 respectively on GEO microarray dataset, far superior than Check Loss and MAE, while retaining the interpretability and robustness of MAE and Check Losses (Table 2 in main text, Tables 5 and 6 (G and H) in appendix).
\par The current limitation of our work is the non-trivial extension of the sBQC loss to the multi-class classification setting since there is no unique way to define the notion of multivariate quantiles. By extension, that would make our proposed sBQC loss much more attractive for application to image and text domains. Furthermore, exploration of the interplay between higher-order differentiability and quantiles is deferred to future work.

\bibliography{main}

\begin{thebibliography}{10}

\bibitem{C&D}
Anonymous.
\newblock Code and data, 2022.

\bibitem{drugGEO}
Jane P.~F. Bai, Alexander~V. Alekseyenko, Alexander Statnikov, I-Ming Wang, and
  Peggy~H. Wong.
\newblock Strategic applications of gene expression: From drug
  discovery/development to bedside.
\newblock {\em The {AAPS} Journal}, 15(2):427--437, January 2013.

\bibitem{geodata}
Tanya Barrett, Stephen~E. Wilhite, Pierre Ledoux, Carlos Evangelista, Irene~F.
  Kim, Maxim Tomashevsky, Kimberly~A. Marshall, Katherine~H. Phillippy,
  Patti~M. Sherman, Michelle Holko, Andrey Yefanov, Hyeseung Lee, Naigong
  Zhang, Cynthia~L. Robertson, Nadezhda Serova, Sean Davis, and Alexandra
  Soboleva.
\newblock {NCBI GEO: archive for functional genomics data sets—update}.
\newblock {\em Oxford Academic}, 41(D1):D991--D995, November 2012.

\bibitem{LBFGS}
Richard~H. Byrd, Peihuang Lu, Jorge Nocedal, and Ciyou Zhu.
\newblock A limited memory algorithm for bound constrained optimization.
\newblock {\em {SIAM} Journal on Scientific Computing}, 16(5):1190--1208,
  September 1995.

\bibitem{LCvae}
Pengfei Chen, Guangyong Chen, and Shengyu Zhang.
\newblock Log hyperbolic cosine loss improves variational auto-encoder, 2019.

\bibitem{D-GEX}
Yifei Chen, Yi~Li, Rajiv Narayan, Aravind Subramanian, and Xiaohui Xie.
\newblock Gene expression inference with deep learning.
\newblock {\em Bioinformatics}, 32(12):1832--1839, February 2016.

\bibitem{GX6}
Ameen Eetemadi and Ilias Tagkopoulos.
\newblock Genetic neural networks: an artificial neural network architecture
  for capturing gene expression relationships.
\newblock {\em Bioinformatics}, 35(13):2226--2234, November 2018.

\bibitem{drugGenX}
Nic Fleming.
\newblock Nic fleming; nature(2018).
\newblock \url{https://www.nature.com/articles/d41586-018-05267-x}, 2018.

\bibitem{GX7}
Blaise Hanczar, Farida Zehraoui, Tina Issa, and Mathieu Arles.
\newblock Biological interpretation of deep neural network for phenotype
  prediction based on gene expression.
\newblock {\em {BMC} Bioinformatics}, 21(1), November 2020.

\bibitem{huberRob}
Peter~J. Huber and Elvezio~M. Ronchetti.
\newblock {\em Robust Statistics}.
\newblock Wiley, 2 edition, 2009.

\bibitem{QR2}
S.~R. Jantre, S.~Bhattacharya, and T.~Maiti.
\newblock Quantile regression neural networks: A bayesian approach.
\newblock {\em Journal of Statistical Theory and Practice}, 15(3), June 2021.

\bibitem{QRkoenker}
Roger Koenker and Gilbert Bassett.
\newblock Regression quantiles.
\newblock {\em Econometrica}, 46(1):33, January 1978.

\bibitem{infFunc}
Pang~Wei Koh and Percy Liang.
\newblock Understanding black-box predictions via influence functions, 2020.

\bibitem{tor21}
Brinda Mahadevan, Ronald~D. Snyder, Michael~D. Waters, R.Daniel Benz,
  Raymond~A. Kemper, Raymond~R. Tice, and Ann~M. Richard.
\newblock Genetic toxicology in the 21st century: Reflections and future
  directions.
\newblock {\em Environmental and Molecular Mutagenesis}, 52(5):339--354, April
  2011.

\bibitem{QR1}
Sang~Jun Moon, Jong-June Jeon, Jason Sang~Hun Lee, and Yongdai Kim.
\newblock Learning multiple quantiles with neural networks.
\newblock {\em Journal of Computational and Graphical Statistics},
  30(4):1238--1248, 2021.

\bibitem{nocedal2nd}
Jorge Nocedal and Stephen~J. Wright.
\newblock {\em Numerical Optimization}.
\newblock Springer, 2 edition, 2006.

\bibitem{QR4}
Oscar Hernan~Madrid Padilla, Wesley Tansey, and Yanzhen Chen.
\newblock Quantile regression with deep relu networks: Estimators and minimax
  rates, 2020.

\bibitem{COVIDgenX}
Thai-Hoang Pham, Yue Qiu, Jucheng Zeng, Lei Xie, and Ping Zhang.
\newblock A deep learning framework for high-throughput mechanism-driven
  phenotype compound screening and its application to {COVID}-19 drug
  repurposing.
\newblock {\em Nature Machine Intelligence}, 3(3):247--257, February 2021.

\bibitem{LipGene}
Tejas Prashanth, Snehanshu Saha, Sumedh Basarkod, Suraj Aralihalli, Soma~S
  Dhavala, Sriparna Saha, and Raviprasad Aduri.
\newblock {LipGene}: Lipschitz continuity guided adaptive learning rates for
  fast convergence on microarray expression data sets.
\newblock {\em {IEEE}/{ACM} Transactions on Computational Biology and
  Bioinformatics}, pages 1--1, 2021.

\bibitem{QR3}
Johannes Schmidt-Hieber.
\newblock Nonparametric regression using deep neural networks with relu
  activation function.
\newblock {\em The Annals of Statistics}, 48(4), Aug 2020.

\bibitem{BQR}
Anuj Tambwekar, Anirudh Maiya, Soma Dhavala, and Snehanshu Saha.
\newblock Estimation and applications of quantiles in deep binary
  classification, 2021.

\bibitem{UNsdg}
{UN Dept.of Economic and Social Affairs}.
\newblock The 17 goals | sustainable development.
\newblock \url{https://sdgs.un.org/goals}, 2015.

\bibitem{nonCvx}
Peng Xu, Farbod Roosta-Khorasani, and Michael~W. Mahoney.
\newblock Second-order optimization for non-convex machine learning: An
  empirical study, 2018.

\bibitem{PyHessian}
Zhewei Yao, Amir Gholami, Kurt Keutzer, and Michael Mahoney.
\newblock Pyhessian: Neural networks through the lens of the hessian, 2020.

\bibitem{AdaH}
Zhewei Yao, Amir Gholami, Sheng Shen, Mustafa Mustafa, Kurt Keutzer, and
  Michael~W. Mahoney.
\newblock Adahessian: An adaptive second order optimizer for machine learning,
  2021.

\bibitem{LALR}
Rahul Yedida, Snehanshu Saha, and Tejas Prashanth.
\newblock Lipschitzlr: Using theoretically computed adaptive learning rates for
  fast convergence, 2020.

\bibitem{GX8}
Ye~Yuan and Ziv Bar-Joseph.
\newblock Deep learning for inferring gene relationships from single-cell
  expression data.
\newblock {\em Proceedings of the National Academy of Sciences},
  116(52):27151--27158, December 2019.

\bibitem{natureArt}
Jan Zrimec, Christoph~S. B\"{o}rlin, Filip Buric, Azam~Sheikh Muhammad,
  Rhongzen Chen, Verena Siewers, Vilhelm Verendel, Jens Nielsen, Mats
  T\"{o}pel, and Aleksej Zelezniak.
\newblock Deep learning suggests that gene expression is encoded in all parts
  of a co-evolving interacting gene regulatory structure.
\newblock {\em Nature Communications}, 11(1), December 2020.

\end{thebibliography}
\bibliographystyle{plain}

\newpage
\appendix
\onecolumn
\section{How did the loss function, \texorpdfstring{$\log-\cosh$}{Lg}, come about?}
The above discussion justifies Log cosh as a loss function and provides a theoretical insight into its effectiveness in different settings. However, it still does not argue for its existence/inspiration as a reconstruction loss in the encoder setting. The goal of deep learning is to learn the manifold structure in data (i.e. natural high dimensional data concentrating to a non-linear low dimensional manifold) and the probability distribution associated with the manifold. An autoencoder learns low dimensional data and represents data as a parametric manifold i.e. a piece-wise linear map from latent to the ambient space. In the case of encoders, low-dimensional data is learnt and represented as a parametric manifold, a piecewise linear map from latent to ambient space.\\
\textbf{$log-cosh$(x) in VAE - A distributional insight:}
We define the encoder and decoder as follows:
\begin{itemize}
  \item Encoder  $\varphi$ :$\chi \rightarrow \digamma$ maps $\Sigma$ to its latent representation $D = \varphi(\Sigma)$ homeomorphically.
  \item Decoder $\psi$: $\digamma \rightarrow \varphi$ maps z to reconstruction $\tilde{x} = \psi(z)= \psi\circ\varphi(x)$
\end{itemize}
\[ 
\varphi, \psi = argmin_{\varphi, \psi} \int_{\chi} L(x,\psi\circ\varphi(x)) \,dx 
\]
where $L$ is the loss function, $F$ is the latent space,
$\chi$ is the ambient space and $\Sigma$ is a topological space $\Sigma \subset \bigcup_{\alpha} U_{\alpha} $. We invoke the pseudo-hyperbolic Gaussian below for the construction of distributions. This leads to the reconstruction loss for Variational AutoEncoders (VAEs) which turns out to be our loss function,  $log-cosh(x)$.

\noindent \textbf{Pseudo-Hyperbolic Gaussian:} 
The strategy to generate the pseudo-hyperbolic Gaussian ((Wrapped gaussian distribution $G(\mu,\Sigma)$ on hyperbolic space $\mathbf{H}$)) is as follows:
\begin{itemize}
  \item Sample a $\vec{v}$ from normal distribution N(0,$\Sigma$) defined over $\mathbf{R}^n$.
  \item Interpret $\vec{v}$ as an element of $T_{\mu}\mathbf{H}^n \subset \mathbf{R}^{n+1}$ by rewriting $\vec{v}$ as v=[0,$\vec{v}$].
  \item Parallel transport vector $v$ to $u \in$ $T_{\mu}\mathbf{H}^n \subset \mathbf{R}^{n+1}$ along the geodesic from $\mu_0$ to $\mu$.
  \item Map u to $\mathbf{H}^n$ using $exp(u)= cosh(||u||_L)+sinh(||u||_L)\frac{u}{||u||_L}$
\end{itemize}


\noindent Reconstruction loss is thus $
-\mathbf{E}_{q_{z|x}}log(p_(\theta)(x|z))$. Replacing $p_{\theta}(x|z)$ with pdf of Hyperbolic secant distribution: $=-log(\frac{1}{2}sech(\frac{\pi x}{2}))=log(2cosh(\frac{\pi x}{2}))
=log(cosh(y))$ where $y=\frac{\pi x}{2}$.\\
Since the metric at the tangent space coincides with the Euclidean metric, several distributions can be produced by applying the construction strategy. $log-cosh(x)$ is one of them.

\section{Proof of theorem-1}\label{app:thm1}
\begin{proof}
Consider $J=\Sigma_{i=1}^m\log\cosh(y_i-\theta^Tx_i)$
\begin{align*}
J&=\Sigma_{i=1}^m\log\cosh(y_i-\theta^Tx_i)\\
\frac{\partial J}{\partial\theta_{\alpha}}&=-\Sigma_{i=1}^{m}\tanh(y^{(i)}-\theta^Tx^{(i)})x_{\alpha}^{(i)}\\
\frac{\partial^2 J}{\partial\theta_{\alpha}\theta_{\beta}}&=\Sigma_{i=1}^{m}\text{sech}^2(y^{(i)}-\theta^Tx^{(i)})x^{(i)}_{\alpha}x^{(i)}_{\beta}
\end{align*}
We construct the Hessian as: $H=XDX^T$ (with $D\equiv diag(m\times m)$ and $D_{ii}=\text{sech}^2(y^{(i)}-\theta^Tx^{(i)}$)

Now, for some $u\in R^d$, consider the expression:
\begin{align*}
u^THu&=u^TXDX^Tu\\
     &=\mid\mid D(X^Tu)\mid\mid^2
\end{align*}
Since $D_{ii}>0$, we have $u^THu>0$, and hence the constructed Hessian is positive definite implying the convexity of the $\log-\cosh$.
\end{proof}
\section{Proof of theorem-2}\label{app:thm2}
\begin{thm}
$\log-\cosh$ is Lipschitz, with a Lipschitz constant:
\[
\frac{1}{m}\tanh(g(0)-||y||).\max_ja_j^{[L]}
\]
\end{thm}
\begin{proof}
  It can be shown that a valid Lipschitz constant for a loss function in the \textit{Neural Net} setting may be obtained via the following expression see equation (12) of \cite{LALR}:
\begin{equation*}
    \max_{i,j}|\frac{\partial E}{\partial w_{ij}^{[L]}}|\leq \max_j|\frac{\partial E}{\partial a_j^{[L]}}|.\max_j|\frac{\partial a_j^{[L]}}{\partial z_j^{[L]}}|.\max_j|a_j^{L-1}|
\end{equation*}
The first term $max_j|\frac{\partial E}{\partial a_j^{[L]}}|$, depends on our choice of the loss and is the main term that we will be spending our time analytically computing. The second term: $max_j|\frac{\partial a_j^{[L]}}{\partial z_j^{[L]}}|$, depends on our choice of the activation function, in the case of using ReLU activations (typical for the regression setting), this term can safely be taken to be equal to $1$ (because ReLU only takes on a $0$ or a $1$, and we are taking a supremum over its range), or maybe computed straightforwardly, depending on the activation that we choose to use for our networks. Finally, the third term has to be computed computationally, which is a really straightforward affair (we henceforth, refer to it as $K_z$)
Hence, we now focus our efforts towards deriving an expression for the first term. The process for the same looks as below:

- First define the loss for the final layer:
  \begin{equation*}
      E(\boldsymbol{a}^{[L]})= \frac{1}{m}\log(\cosh(\boldsymbol{a}^{[L]}-\boldsymbol{y}))
  \end{equation*}
- Now we write the derivative:
  \begin{equation*}
      \frac{\partial E}{\partial a^{[L]}}=\frac{1}{m}\tanh(\boldsymbol{a}^{[L]}-\boldsymbol{y})
  \end{equation*}
- Now, we want to find where the parent equation for the Lipschitz constant attains a maximum (and consequently, its maximum value), for which we turn to its second derivative(and points where it vanishes)
  \begin{align*}
    \frac{\partial E^2}{\partial^2 \boldsymbol{a}^{[L]}}&=\frac{1}{m}\sech^2(\boldsymbol{a}^{[L]}-\boldsymbol{y})\\
    \frac{\partial E^2}{\partial^2 w_{ij}^{[L]}}&=\frac{1}{m}\sech^2(\boldsymbol{a}^{[L]}-\boldsymbol{y}).K_z\\
    \frac{\partial E^2}{\partial^2 w_{ij}^{[L]}}&=0: \text{For computing the maximum}
  \end{align*}
- Because $\sech$ remains non-zero, the only time the above second derivative vanishes is when $K_z=0$, i.e. as per our definition, $a_i^{[L-1]}=0$, which in turn implies, $z^{[L]}=w_{ij}^{[L]}a_i^{[L-1]}=0$, hence finally yielding: $a^{[L]}=g(z^{[L]})=g(0)$, finally giving us $\frac{\partial E}{\partial a^{[L]}}=\frac{1}{m}\tanh(g(0)-\boldsymbol{y})$ (here, $g$, is the activation function)
  Now, notice that (can be seen algebraically for the \textit{L-1} norm):
  \begin{equation*}
    ||\tanh(x)||=\tanh(||x||)
  \end{equation*}
  The above, coupled with the simple triangular inequality for the \textit{2-norm}(i.e. imposing the \textit{2-norm} norm on both sides of the derivative equation) obtain:
  \begin{equation*}
      \frac{\partial E}{\partial a_j^{[L]}}\leq\frac{1}{m}\tanh(g(0)-\mid\mid\boldsymbol{y}\mid\mid)\\
  \end{equation*}
  \begin{equation}
      \boxed{\max_{i,j}|\frac{\partial E}{\partial w_{ij}^{[L]}}|=\frac{1}{m}\tanh(g(0)-||\boldsymbol{y}||).K_z}
  \end{equation}
  Where $||\boldsymbol{y}||$ is the maximum norm (across batches) of the labels.
\end{proof}

\section{Proof of Theorem-3}\label{app:thm3}
\begin{thm}
The \textbf{Smooth Binary Quantile Classification Loss} derived from the $log-cosh$ is:
\begin{align*}
L(y_i, \hat{y_i})&= y_i\log(\hat{p_i})+(1-y_i)\log(1-\hat{p_i})\\
\hat{p_i}&=1-F_{\tau}(\hat{y_i})\\
\hat{y_i}&=f_{\tau}(x_i) \quad \text{ where, $f_{\tau}$, is the latent function}
\end{align*}
For any real-valued random variable $Z$, with distribution function $F(z)$, with $F(z) = P(Z \le z) $, the quantile function $Q(\tau)$ is given as
$Q(\tau) = F^{-1}(\tau) = \inf\{r: F(r) \ge \tau\} $ for any $0 < \tau < 1 $. Define $\tau$ as the marker for a typical Quantile loss. Then, the CDF, $F(\cdot)$ of $f$ assumes the closed-form expression:
\[F(x)= \begin{cases}
\tau+\frac{4\tau}{\pi}\arctan(\tanh(\frac{x}{2})) & x\leq 0\\
\tau+\frac{4(1-\tau)}{\pi}\arctan(\tan(\frac{x}{2})) & x>0
   \end{cases}\]
\end{thm}
\begin{proof}

\par Let $L(x)$ be the loss function with $x$ being the input to the loss function. Then for the symmetric version of the loss function, 

\begin{equation}
    L(x) = log(cosh(x))
\end{equation}

\[
    L(x)= 
\begin{cases}
    log_e(\frac{(e^x+e^{-x})}{2})  & x\ge 0\\
    log_e(\frac{(e^x+e^{-x})}{2})  & x < 0
\end{cases}
\]

After correction, the loss function becomes
\begin{equation*}
    L(x) = log(cosh(x)) + log(2)
\end{equation*}


Let $f(x)$ be the probability density function (PDF) and $F(x)$ be the cumulative density function (CDF). Then,

\begin{equation*}
    f(x) \propto e^{-L(x)}
\end{equation*}

We know that 
\begin{equation*}
    \int_{-\infty}^{\infty} f(x) = 1
\end{equation*}

Since 
\begin{align*}
    \int_{}^{} \frac{1}{cosh(x)} \,dx = 2*tan^{-1}(tanh(\frac{x}{2}))
\end{align*}

\begin{align*}
    \int_{-\infty}^{\infty} \frac{1}{cosh(x)} \,dx = \pi 
\end{align*}

The PDF and CDF are obtained to be:

\begin{equation*}
    f(x) = \frac{1}{\pi} e^{-log(cosh(x))}
\end{equation*}

\begin{equation*}
    F(x) = \frac{\pi}{2} + 2*tan^{-1}(tanh(\frac{x}{2}))
\end{equation*}

The asymmetric version of the loss function $L(x)$ is known to be:

\[
    L(x)= 
\begin{cases}
    (1 - \tau) * log(cosh(x))  & x< 0\\
    \tau * log(cosh(x))  & x \ge 0
\end{cases}
\]

Let $f(x)$ be the probability density function (PDF) and $F(x)$ be the cumulative density function (CDF). Then,

\begin{equation*}
    f(x) \propto e^{-L(x)}
\end{equation*}

\begin{equation*}
    =\int_{-\infty}^{0} \frac{1}{(1 - \tau) * log(cosh(x))} \,dx +  \int_{0}^{\infty} \frac{1}{(\tau) * log(cosh(x))}
\end{equation*}

\begin{equation*}
    =\frac{1}{(1-\tau)} * \frac{\pi}{2} + \frac{1}{\tau} * \frac{\pi}{2}
\end{equation*}

\begin{equation*}
    =\frac{\pi}{2} * (\frac{1}{\tau*(1-\tau)})
\end{equation*}

The PDF is obtained to be:
\begin{equation*}
    f(x) = \frac{2*\tau*(1-\tau)}{\pi} (\frac{\mathbbm{1}{x<0}}{(1-\tau)cosh(x)} + \frac{\mathbbm{1}{x \ge 0}}{\tau*cosh(x)})
\end{equation*}

\begin{equation*}
    = \frac{2}{\pi*cosh(x)} (\tau * (\mathbbm{1}{x<0}) + (1 - \tau)* (\mathbbm{1}{x \ge 0}))
\end{equation*}

Verifying the correctness of this by substituting $\tau$=0.5,

\begin{equation*}
    \int_{}^{} f(x) \,dx = \frac{1}{\pi} \int_{-\infty}^{\infty} cosh(x)\,dx = 1
\end{equation*}

Calculating the CDF F(x) separately for the two cases:

For  $x<0$

\begin{equation*}
    F(x) = \frac{2 \tau}{\pi} \int_{-\infty}^{x} \frac{1}{cosh(x)} dx= \frac{2\tau}{\pi}(\frac{\pi}{2} + 2tan^{-1}(tanh(\frac{x}{2})))
\end{equation*}

\begin{equation*}
    = \tau + \frac{4\tau}{\pi}tan^{-1}(tanh(\frac{x}{2}))
\end{equation*}

For $x \ge 0$

\begin{equation*}
    F(x) = \tau + \int_{0}^{x} (1-\tau) \frac{2}{\pi} \frac{1}{cosh(x)} \,dx
\end{equation*}

\begin{equation*}
    = \tau + \frac{2(1-\tau)}{\pi}(2tan^{-1}(tanh(\frac{x}{2}) - 0)
\end{equation*}

\begin{equation*}
    = \tau + \frac{4(1-\tau)}{\pi}tan^{-1}(tanh(\frac{x}{2}))
\end{equation*}
\end{proof}
Note: $\tau$ is the marker for quantile losses which can be extended from logcos{h(x)}, making this loss interpretable as well. For example, $\tau=0.5$ gives us the median i.e. MAE and $L, L', L"$ are well-defined for $\tau=0.5$. Thus, at $\tau=0.5$, the smooth, quantiled version of MAE i.e logcos{h(x)} is continuous and twice differentiable and interpretable in the sense of binary quantile regression and binary quantile classification \cite{BQR}

\section{Proof of Theorem-4}\label{app:thm4}
\begin{thm}
The Binary Smooth Quantile Classification Loss has the Lipschitz constant:
\[
\frac{2}{\pi} \max(1, \frac{1-\tau}{\tau}, \frac{\tau}{\tau-1})
\]
\end{thm}
\begin{proof}
We may write the sBQC loss as:
\[
L(y,z)=-(y\log p_z+(1-y)\log p_z)
\]

Where,
\[p_z\equiv\begin{cases}
1-\tau-\frac{4\tau}{\pi}\arctan(\tanh(\frac{z}{2})) & z\leq 0\\
1-\tau-\frac{4(1-\tau)}{\pi}\arctan(\tanh(\frac{z}{2})) & z>0
\end{cases}
\]

We define: $\Delta L(y)\equiv\frac{\mid L(y,z_2)-L(y,z_1)\mid}{\mid z_2-z_1\mid}$ -- we may easily break down the computation in the case of Binary Classification problems into several pieces, which we deal with on a case-by-case basis as follows:

\textit{Case-1a}: $0<z_1<z_2, y=1$:
\begin{align*}
\Delta_z L(1)=\frac{
\begin{aligned}
    \log(1&-\tau-4(\frac{1-\tau}{\pi})\arctan(\tanh(\frac{z_2}{2}))-\\
    &\log(1-\tau-4(\frac{1-\tau}{\pi})\arctan(\tanh(\frac{z_1}{2}))
\end{aligned}
}{z_2-z_1}
\end{align*}
The RHS in the above expression will take up a maximum value (which is what we want in the case of a Lipschitz constant) for $z_2,z_1\to 0$, first imposing $z_1\to 0$.
\[
\lim_{z_2\to 0}\Delta_z L(1)=\frac{\log(1-\tau-4(\frac{1-\tau}{\pi})\arctan(\tanh(\frac{z_2}{2}))-\log(1-\tau)}{z_2}
\]
Which finally reduces to:
\[
\lim_{z_2,z_1\to 0}\Delta_z L(1)=-\frac{2}{\pi}
\]

\textit{Case-1b}: $0<z_1<z_2,y=0$,
Following the same structure as above, we get:
\[
\lim_{z_2\to 0}\Delta_z L(0)=\frac{\log(\tau+4(\frac{1-\tau}{\pi})\arctan(\tanh(\frac{z_2}{2}))-\log{\tau}}{z_2}
\]
Which reduces to:
\[
\lim_{z_2,z_1\to 0}\Delta_z L(0)= \frac{2-2\tau}{\pi\tau}
\]

\textit{Case-2a}: $z_1<0<z_2,y=1$,
\[
\lim_{z_2\to 0}\Delta_z L(1)=\frac{\log(1-\tau-4(\frac{1-\tau}{\pi})\arctan(\tanh(\frac{z_2}{2}))-\log(1-\tau)}{z_2}
\]
Which reduces to:
\[
\lim_{z_2,z_1\to 0}\Delta_z L(0)= \frac{-2}{\pi}
\]

\textit{Case-2b}: $z_1<0<z_2,y=0$,
\[
\lim_{z_2\to 0}\Delta_z L(0)=\frac{\log(\tau+4(\frac{1-\tau}{\pi})\arctan(\tanh(\frac{z_2}{2}))-\log{\tau}}{z_2}
\]

Which reduces to:
\[
\lim_{z_2,z_1\to 0}\Delta_z L(0)= \frac{2-2\tau}{\pi\tau}
\]

\textit{Case-3a}: $z_1<z_2<0,y=1$,
\[
\lim_{z_2\to 0}\Delta_z l(1)=\frac{\log(1-\tau-4(\frac{\tau}{\pi})\arctan(\tanh(\frac{z_2}{2}))-\log(1-\tau)}{z_2}
\]

which reduces to:
\[
\lim_{z_2,z_1\to 0}\Delta_z l(1)= \frac{2\tau}{\pi(\tau-1)}
\]

\textit{Case-3b}: $z_1<z_2<0,y=0$,
\[
\lim_{z_2\to 0}\Delta_z L(0)=\frac{\log(\tau+4(\frac{\tau}{\pi})\arctan(\tanh(\frac{z_2}{2}))-\log(\tau)}{z_2}
\]

which reduces to:
\[
\lim_{z_2,z_1\to 0}\Delta_z l(1)= \frac{2}{\pi}
\]

Hence, cumulatively, we may write the Lipschitz constant of the sBQC loss as:

\[
\max(\frac{2}{\pi},\frac{2-2\tau}{\pi\tau},\frac{2\tau}{\pi(\tau-1)})
\]
\end{proof}
\section{Robustness to label noise}\label{app:robust}
Another useful property of $\log-\cosh$ as a loss function is that it is \emph{robust to label noise}, meaning that the fluctuations in function output under small perturbations in its input have an upper bound. This places $\log-\cosh$ on an equivalent setting as MAE. We state it formally below:
\begin{lemma}[Robustness to label-noise]
$\log-\cosh$ is robust to label noise.
\end{lemma}
\begin{proof}
Mathematically, label noise resilience is captured as:
$||f(x+\epsilon)-f(x)|| \to$ 0 as $\epsilon\to0$ for some $\epsilon>0$\\
Now, in our case: $f(x)=\log\cosh{x}\text{; }f(x+\epsilon)=\log\cosh{(x+\epsilon)}$. Consider:
\begin{align*}
    ||f(x+\epsilon)-f(x)||&=\mid\mid\log\cosh{(x+\epsilon)}-\log\cosh{x}\mid\mid\\
                                      &=\mid\mid\log{(\frac{\cosh{(x+\epsilon)}}{\cosh{x}})}\mid\mid\\
                                      &=\mid\mid\log{(\cosh{\epsilon}+\tanh{x}.\sinh{\epsilon})}\mid\mid\\
\end{align*}
From the above, clearly, as $\epsilon\to 0$, $\mid\mid f(x+\epsilon)-f(x)\mid\mid\to\delta$ for some $\delta>0; \delta\leq \epsilon$, and hence the theorem is established.
\end{proof}
\end{document}